\documentclass[letterpaper,10pt,conference]{ieeeconf}

\IEEEoverridecommandlockouts
\usepackage{amsmath,amssymb,amsfonts}
\usepackage{bm}
\usepackage{graphicx}
\usepackage{booktabs}
\usepackage{cite}
\usepackage{algorithm}
\usepackage{algpseudocode}
\usepackage{xcolor}
\usepackage[normalem]{ulem}

\newtheorem{assumption}{Assumption}
\newtheorem{theorem}{Theorem}
\newtheorem{remark}{Remark}
\newtheorem{definition}{Definition}

\newtheorem{proposition}{Proposition}

\title{\LARGE \bf
Learning Dynamic Games for Cooperative Iterative Tasks
}

\author{Mark Pustilnik$^{1}$, Francesco Borrelli$^{1}$% <-this % stops a space
\thanks{$^{1}$University of California at Berkeley, USA \quad {\tt\small \{pkmark, fborrelli\}@berkeley.edu}}%
}

\begin{document}

\maketitle
\thispagestyle{empty}
\pagestyle{empty}
%===============================================================================
\begin{abstract}

We consider cooperative and iterative tasks involving multiple agents whose interactions are modeled as a dynamic game with shared constraints. We propose a general learning framework applicable to nonlinear and nonconvex dynamic games that uses previously executed feasible trajectories to construct a sampled joint safe set and player-specific cost-to-go functions. At each time step, the agents jointly select a previously sampled terminal state and a Generalized Nash equilibrium used to reach it. A new solution is adopted only if it satisfies a cooperative acceptance rule agreed upon by the players. The framework accommodates individual non-degradation requirements and selection criteria such as weighted costs and Nash bargaining. We show that the proposed framework guarantees recursive feasibility and characterize conditions under which each player's cost is non-increasing across iterations.

\end{abstract}

%===============================================================================
\section{Introduction}
Many robotic systems repeatedly perform cooperative tasks from identical initial conditions, including payload transport and coordinated assembly using multiple manipulators. Previous executions provide trajectory data that can improve subsequent performance. Although
the agents cooperate, each may retain its own dynamics, constraints, and objective, while shared constraints couple their decisions. Dynamic games model these interactions. Learning Model Predictive Control (LMPC) exploits repeated single agent tasks by constructing terminal safe sets and cost-to-go functions from successful trajectories, enabling improvement while preserving recursive feasibility \cite{RosoliaBorrelli2018}. Extending this principle to dynamic games requires addressing multiple equilibria with different player trade-offs: improving one player's cost may worsen another's. Learning must therefore determine both how to improve previously executed trajectories and how the players agree on which new solution to adopt. In this work, we consider cooperative iterative tasks in which the players agree in advance on rules for selecting new solutions. 
Cooperation does not require a common objective: each player retains its individual cost function, while candidate solutions are evaluated according to a jointly agreed criterion. Individual non-degradation can determine admissibility, while weighted costs or Nash bargaining (NB) select among admissible candidates. We propose a general learning framework for dynamic games where each agent uses previously executed feasible joint trajectories to construct a sampled joint safe set. At each time step, candidate solutions connect the current joint state to a previously visited joint state, after which the remainder of a trajectory from a previous successful iteration can be reused. Since the terminal state is selected from a finite set of previously sampled states, the learning problem naturally includes a discrete decision layer in addition to the continuous dynamic-game equilibrium computation. 

The framework accommodates nonlinear dynamics and nonconvex costs and constraints, starting from an initial successful feasible trajectory. The resulting performance depends on the initial trajectory, sampled states, and candidates returned by the game solver. The contributions of this paper are:
\begin{enumerate}
\item A learning framework for cooperative iterative dynamic games with shared constraints, jointly selecting terminal occurrences and equilibrium candidates using a pre-agreed cooperative rule.
\item Recursive-feasibility and individual non-degradation guarantees under stated conditions, supported by an explicitly retained feasible continuation.
\item A nonlinear, nonconvex three-player demonstration comparing equilibrium exploration and cooperative selection under shared collision-avoidance constraints.
\end{enumerate}

%===============================================================================
\section{Related Work}
Learning from repeated tasks has been studied through Iterative Learning Control (ILC) \cite{Bristow2006,Wang2009}, including robust formulations that account for non-repetitive disturbances \cite{LinSunTomizuka2015}, as well as through Batch-MPC \cite{LeeChinLeeLee1999,LeeLee2000}. Extensions of LMPC have considered simultaneous improvement of multiple objectives \cite{NairVallonBorrelli2025}, decentralized multi-agent control under shared constraints \cite{Zhu2020}, and non-cooperative distributed MPC with iterative learning \cite{Hu2020}. In games with coupled feasible sets, generalized Nash equilibria (GNEs) provide the corresponding solution concept. Normalized equilibria impose a common structure on the shared-constraint multipliers across players \cite{Rosen1965,FacchineiKanzow2010}, whereas non-normalized equilibria enlarge the set of candidate solutions by allowing more general multiplier allocations \cite{PustilnikBorrelli2025}. Learned terminal information has also been incorporated into generalized Nash games \cite{Kim2025}, while cooperative criteria such as Nash bargaining can be used to compare candidate outcomes \cite{Nash1950}.

Iterative learning exploits repeated executions to improve performance, while equilibrium selection addresses trade-offs among players. This work combines these perspectives through cooperative selection of terminal occurrences and equilibrium candidates using stored joint trajectories. Feasible continuations and player-wise acceptance support recursive feasibility and individual non-degradation, even when candidate generation is incomplete. The retained backup need not itself be an equilibrium; equilibrium selection determines which candidate trade-offs are explored.

%===============================================================================
\section{Problem Formulation} \label{sec:problem_formulation}

We consider a cooperative task that is executed repeatedly. Each execution, referred to as an \emph{iteration}, starts from the same initial joint state and terminates when the players reach a common target joint state. The iterations are indexed by $j$, and time within an iteration is indexed by $k$. The task is modeled as a discrete-time dynamic game with players $i \in \mathcal{M} := \{1, \ldots, M\}$, deterministic and exactly known dynamics, and game data that remain unchanged across iterations. At time $k$ of iteration $j$, the joint state and input are
%Consider a discrete-time dynamic game with players $i\in\mathcal M:=\{1,\ldots,M\}$, deterministic exactly known dynamics, and game data unchanged across iterations. Let $j$ index iterations and $k$ time, with joint state and input
\[
x_k^j:=\operatorname{col}(x_{1,k}^j,\ldots,x_{M,k}^j),
\qquad
u_k^j:=\operatorname{col}(u_{1,k}^j,\ldots,u_{M,k}^j).
\]
Every iteration starts from $x_0^j=x_S$ and seeks to reach $x_F:=\operatorname{col} (x_{1,F},\ldots,x_{M,F})$. For an iteration completed in finite time, define 
\begin{equation}
    T^j:=\min\{k\ge0:x_k^j=x_F\}.
\end{equation}

\begin{assumption}[Nonnegative Stage Costs] \label{ass1}
For every player $i$, the stage cost satisfies 
\begin{equation}
    \ell_i(x,u)\ge0,\qquad \ell_i(x_F,u_F)=0
    \label{eq:nonnegative_stage_cost}
\end{equation}
for all feasible $(x,u)$, where $u_F$ is a feasible terminal input that
maintains the system at $x_F$.
\end{assumption}

At time $k$ of iteration $j$, given a prediction horizon $N$, player $i$ solves
\begin{subequations}\label{eq:dynamic_game}
\begin{align}
\min_{\mathbf{x}_{i,k}^j,\mathbf{u}_{i,k}^j} \quad
     &J_{i,k}^j :=
    \sum_{\tau=k}^{k+N-1}
    \ell_i(x_{\tau|k}^j,u_{\tau|k}^j)
    + V_i(x_{k+N|k}^j)
    \label{eq:dynamic_game_cost} \\[1mm]
\text{s.t.} \quad
    & x_{i,k|k}^{j}=x_{i,k}^{j}
    \label{eq:dynamic_game_initial}\\
    & x_{i,\tau+1|k}^j = f_i(x_{i,\tau|k}^j,u_{i,\tau|k}^j),~ \tau=k,...,k+N-1
    \label{eq:dynamic_game_dynamics}\\
    & g_i(x_{i,\tau|k}^j,u_{i,\tau|k}^j)\leq 0,  ~ \tau=k,...,k+N-1
    \label{eq:dynamic_game_private}\\
    & s(x_{\tau|k}^j,u_{\tau|k}^j)\leq 0, ~ \tau=k,...,k+N-1
    \label{eq:dynamic_game_shared}
\end{align}
\end{subequations}

Here, $\ell_i$ and $V_i$ are player $i$'s stage and terminal costs, which may depend on the joint state and input (no objective separability is assumed). The functions $f_i$, $g_i$, and $s$ describe individual dynamics, private constraints, and shared constraints, respectively. Additive cost terms independent of a player's own strategy do not affect its best response, but remain part of its realized cost and stored cost-to-go and can therefore influence cooperative candidate selection.

\subsection{Generalized Nash Equilibrium}

Let $\mathbf z_{i,k}^j:= (\mathbf x_{i,k}^j,\mathbf u_{i,k}^j)$ denote player $i$'s predicted strategy, and let $\mathcal Z_{i,k}^j(\mathbf z_{-i,k}^j)$ be its feasible set defined by \eqref{eq:dynamic_game_initial}-\eqref{eq:dynamic_game_shared} where the subscript $-i$ denotes all players except $i$. A joint strategy $\mathbf z_k^{j,\star}$ is a generalized Nash equilibrium (GNE) if
\begin{equation}
\mathbf z_{i,k}^{j,\star}
\in
\underset{
\mathbf z_{i,k}^j\in
\mathcal Z_{i,k}^j(\mathbf z_{-i,k}^{j,\star})
}{\arg\min}
J_{i,k}^j
\bigl(\mathbf z_{i,k}^j,\mathbf z_{-i,k}^{j,\star}\bigr),
\quad \forall i\in\mathcal M.
\label{eq:GNE}
\end{equation}
Thus, no player can lower its cost through a feasible unilateral action. For differentiable games, equilibrium candidates can be computed using the players' combined KKT conditions \cite{Rosen1965,FacchineiKanzow2010}. For compactness, let $h_i(\mathbf z_i)=0$ collect player $i$'s dynamics and other equality constraints over the prediction horizon, let
$g_i(\mathbf z_i)\leq0$ collect its private inequality constraints, and let $s(\mathbf z)\leq0$ collect the shared constraints. The Lagrangian of player $i$ is then
\begin{equation}
\begin{aligned}
    \mathcal{L}_i
    =\;&
    J_{i,k}^{j}
    + \mu_i^\top h_i
    + \lambda_i^\top g_i
    + \sigma_i^\top s ,
\end{aligned}
\label{eq:player_lagrangian}
\end{equation}
where $\mu_i$, $\lambda_i$, and $\sigma_i$ are the dual variables associated with the equality, private inequality, and shared inequality constraints, respectively. Under a suitable constraint qualification (see \cite{FacchineiKanzow2010} for details), every GNE satisfies the combined KKT conditions: 
\begin{subequations}\label{eq:GNE_KKT}
\begin{align}
    & \nabla_{\mathbf{z}_i}\mathcal{L}_i = 0,
    && \forall i\in\mathcal{M},\\
    & h_i(\mathbf{z}_i)=0,
    && \forall i\in\mathcal{M},\\
    & 0\leq \lambda_i \perp -g_i(\mathbf{z}_i)\geq0,
    && \forall i\in\mathcal{M},\\
    & 0\leq \sigma_i \perp
    -s(\mathbf{z}_1,\ldots,\mathbf{z}_M)\geq0,
    && \forall i\in\mathcal{M}.
\end{align}
\end{subequations}
These conditions are sufficient when each player's problem is convex in its own strategy, otherwise they identify stationary candidates that need not be GNEs. The dynamic game may admit multiple GNEs. The learning framework developed in this work does not require the equilibrium to be unique. 

\subsection{Cooperative Iterative Learning Objective}
For an iteration completed at time $T^j$, player $i$'s realized cost is \begin{equation}
    \mathcal J_i^j
    :=
    \sum_{k=0}^{T^j-1}\ell_i(x_k^j,u_k^j)
    +V_i(x_{T^j}^j).
    \label{eq:realized_iteration_cost}
\end{equation}
Starting from an initial successful feasible trajectory from $x_S$ to $x_F$, the objective is to use previous executions to improve performance while ensuring: 
\begin{enumerate}
    \item \emph{Recursive feasibility:} a feasible continuation to $x_F$ remains available at every time step.

    \item \emph{Cooperative acceptance:} new strategies satisfy the players' pre-agreed admissibility rule.

    \item \emph{Individual non-degradation:} under the corresponding acceptance rule, successive completed iterations satisfy
    \begin{equation}
        \mathcal J_i^j\le\mathcal J_i^{j-1},
        \qquad \forall i\in\mathcal M.
        \label{eq:iteration_nondegradation}
    \end{equation}
\end{enumerate}
In the learning formulation below, each GNE is an open-loop equilibrium of the finite-horizon game conditional on the selected terminal joint state. This does not imply that the concatenated full-task trajectory or the resulting receding-horizon policy is a Nash equilibrium.

%===============================================================================
\section{Learning Dynamic Game Method}
At each time step, the players jointly select a terminal joint state from the sampled safe set and an equilibrium candidate that reaches it. These two choices are coupled: each terminal state defines a distinct terminal-constrained game, and hence its own set of equilibria, which may be infinite. 
Each equilibrium candidate is an $N$-step predicted joint trajectory ending at the selected terminal state. Appending to it the stored trajectory from that terminal state to $x_F$ yields a complete joint trajectory from the current state to the target, which is evaluated under the agreed admissibility and selection rules. If no candidate is accepted, the players continue along the retained feasible backup.
%Each candidate is concatenated with the stored continuation and evaluated using the agreed admissibility and selection rules. If no candidate is accepted, the previously retained feasible backup is used.

\subsection{Sampled Joint Safe Set and Cost-to-Go}
Terminal safe sets provide feasible continuations beyond the prediction horizon, while terminal cost-to-go information accounts for the remaining task cost. Computing these quantities directly can be difficult for nonlinear, constrained systems. Following LMPC \cite{RosoliaBorrelli2018}, we construct a sampled joint safe set and player-specific costs-to-go from previously successful trajectories. An iteration is successful if it reaches $x_F$ in finite time while satisfying all
private and shared constraints. Let $\mathcal I^j$ denote the successful iterations available up to iteration $j$, with $T^l$ the completion time of iteration $l$. The sampled joint safe set is
\begin{equation}
    \mathcal{SS}^{j} := \bigcup_{l\in\mathcal{I}^{j}} \bigcup_{t=0}^{T^{l}} \left\{x_t^{l}\right\}
    \label{eq:sampled_safe_set}
\end{equation}
Every state in $\mathcal{SS}^j$ has a known feasible continuation to $x_F$. For each sampled occurrence $x_t^l$, player $i$'s realized cost-to-go is 
\begin{equation}
    Q_{i,t}^{l} := \sum_{\tau=t}^{T^{l}-1} \ell_i(x_\tau^{l},u_\tau^{l}) +     V_i(x_{T^{l}}^{l})
    \label{eq:sampled_cost_to_go}
\end{equation}

To associate each occurrence with its continuation and cost vector, define the trajectory database
\begin{equation}
    \mathcal D^j := \left\{ (l,t)\;\middle|\; l\in\mathcal I^j,\quad t=0,\ldots,T^l \right\}.
    \label{eq:sampled_database}
\end{equation}

Each entry $(l,t)\in\mathcal D^j$ identifies the joint state $x_t^l$, the costs $Q_{i,t}^l$, and the stored state-input suffix from that state to $x_F$. Repeated visits to the same joint state remain distinct entries because their continuations and cost vectors may differ. Reconnecting a feasible predicted trajectory to an entry therefore provides a complete feasible continuation, without requiring a new equilibrium to be found at every subsequent step. 

\begin{remark}[Retained Feasible Backup]
An executed trajectory may combine segments generated using different terminal states and GNE parameterizations. Its continuation therefore need not be a GNE for any single parameterization and may not be recovered by resolving a dynamic game. The method explicitly retains the feasible state-input continuation as a backup.
\end{remark}
\subsection{Mixed-Integer Dynamic Game Formulation}

At time $k$ of iteration $j$, the players select a terminal occurrence from $\mathcal D^{j-1}=\{q_1,\ldots,q_{n_s}\}$, where $q_r=(j_r,k_r)$ identifies a stored joint state, its continuation, and player-specific costs-to-go. The common binary vector $\boldsymbol\delta=(\delta_1,\ldots,\delta_{n_s})$ satisfies 

\begin{equation}
    \sum_{r=1}^{n_s}\delta_r=1,
    \qquad \delta_r\in\{0,1\},
    \label{eq:terminal_binary_sum}
\end{equation}
selecting exactly one occurrence to be reached after
$N$ prediction steps. For fixed $\boldsymbol\delta$,
player $i$ solves

\begin{subequations}\label{eq:mixed_integer_game}
\begin{align}
\min_{\mathbf{x}_{i,k}^j,\mathbf{u}_{i,k}^j}
&~
J_{i,k}^j(\boldsymbol{\delta}) = \sum_{\tau=k}^{k+N-1} \ell_i \left(x_{\tau|k}^j,u_{\tau|k}^j\right) + \sum_{r=1}^{n_s} \delta_r Q_{i,k_r}^{j_r}
\label{eq:mig_cost}
\\[1mm]
\text{s.t.} \quad
&
x_{i,\tau+1|k}^j
=
f_i
\left(
x_{i,\tau|k}^j,
u_{i,\tau|k}^j
\right),
~
\tau=k,\ldots,k+N-1
\label{eq:mig_dynamics}
\\
&
x_{i,k|k}^j=x_{i,k}^j
\label{eq:mig_initial}
\\
&g_i\left(x_{i,\tau|k}^j,u_{i,\tau|k}^j\right) \leq 0, ~ \tau=k,...,k+N-1 
\label{eq:mig_private} \\
&s \left(x_{\tau|k}^j,u_{\tau|k}^j\right) \leq 0, ~ \tau=k,...,k+N-1 
\label{eq:mig_shared}\\
&x_{i,k+N|k}^j = \sum_{r=1}^{n_s} \delta_r x_{i,k_r}^{j_r}
\label{eq:mig_terminal}
\end{align}
\end{subequations}

The objective combines the predicted horizon cost with the selected continuation cost, while \eqref{eq:mig_terminal} enforces exact reconnection to the stored joint state. For fixed $\boldsymbol\delta$, the continuation cost is constant in the inner game, but it influences what terminal joint state is eventually selected. In the ideal formulation, a candidate is a pair
\begin{equation}
    \left(
    \boldsymbol{\delta},
    \mathbf z_k^{j,\star}(\boldsymbol{\delta})
    \right),
    \label{eq:candidate_pair}
\end{equation}
where $\boldsymbol\delta$ satisfies \eqref{eq:terminal_binary_sum} and $\mathbf z_k^{j,\star}(\boldsymbol\delta)$ is a selected GNE of \eqref{eq:mixed_integer_game}. Terminal selection is discrete, while the corresponding strategies are continuous. Since each terminal selection may admit multiple GNEs, the cooperative rule selects among both terminal occurrences and equilibrium candidates.

\subsection{GNE Parameterization}

For a given terminal-state selection, the corresponding dynamic game may admit multiple GNEs. We explore equilibrium candidates using the shared-multiplier parameterization of \cite{PustilnikBorrelli2025}. For a common multiplier $\sigma\in\mathbb R_+^{m_s}$, define 
\begin{equation}
    A_i=\operatorname{diag}(a_{i,1},\ldots,a_{i,m_s}),
    \qquad a_{i,r}>0,
    \label{eq:scaling_matrix}
\end{equation}
and player-specific shared multipliers 
\begin{equation}
    \sigma_i=A_i\sigma.
    \label{eq:scaled_shared_multiplier}
\end{equation}
For each fixed terminal selection, the parameterized Lagrangian is ($h_i$ also includes the terminal equality~\eqref{eq:mig_terminal}).
\begin{equation}
    \mathcal L_i
    =
    J_{i,k}^j(\boldsymbol\delta)
    +\mu_i^\top h_i
    +\lambda_i^\top g_i
    +(A_i\sigma)^\top s.
    \label{eq:scaled_lagrangian}
\end{equation}
For fixed $\mathcal A:=\{A_1,\ldots,A_M\}$, candidates are computed from the KKT conditions:
\begin{subequations}\label{eq:scaled_kkt}
\begin{align}
    & \nabla_{\mathbf z_i}\mathcal L_i=0,
    && \forall i\in\mathcal M\\
    & h_i(\mathbf z_i)=0,
    && \forall i\in\mathcal M\\
    & 0\le\lambda_i\perp-g_i(\mathbf z_i)\ge0,
    && \forall i\in\mathcal M\\
    & 0\le\sigma\perp-s(\mathbf z_1,\ldots,\mathbf z_M)\ge0
\end{align}
\end{subequations}
Positive diagonal scaling preserves shared-constraint complementarity while allowing different player multipliers. Varying $\mathcal A$ explores candidate equilibria, although a solution need not exist for every allocation. In nonconvex games, KKT solutions remain stationary candidates and do not certify GNEs. To reduce the parameter dimension, we restrict
\begin{equation}
    A_i=a_iI,\qquad \sigma_i=a_i\sigma,\qquad a_i>0.
    \label{eq:scalar_scaling}
\end{equation}
A common rescaling of all $a_i$ can be absorbed into $\sigma$, so imposing
\begin{equation}
    \sum_{i=1}^{M}a_i=1
    \label{eq:scalar_normalization}
\end{equation}
removes this redundancy and leaves $M-1$ independent allocation parameters. This scalar parameterization recovers the KKT structure of Rosen's $r$-normalized equilibria \cite{Rosen1965}, with $r_i\propto1/a_i$.

\subsection{Cooperative Solution Selection}

At each time $k$ of iteration $j$, the players must decide whether to replace the retained backup and, if so, with which candidate. Because each player keeps its own cost function, candidates cannot in general be ranked by a single objective. Instead, before the task begins, the players agree on two rules:
\begin{enumerate}
    \item an \emph{admissibility rule}, which determines the candidates acceptable to every player, and
    \item a \emph{selection rule}, which chooses one candidate among the admissible ones.
\end{enumerate}
The selected candidate becomes the new backup. If no candidate is admissible, the players continue along the current backup.

\textit{Candidate costs:} Let $\mathcal{C}^j_k$ denote the set of primal-feasible candidates computed at time $k$ of iteration $j$. Although the database $\mathcal{D}^{j-1}$ is finite, continuous multiplier allocations and equilibrium non-uniqueness may yield infinitely many candidates. Each candidate $c \in \mathcal{C}^j_k$ selects a terminal occurrence $q_r = (j_r, k_r)$ and, together with the stored trajectory from $x^{j_r}_{k_r}$ to $x_F$, defines a complete joint trajectory from the current state to the target. Player $i$'s cost along this trajectory is
\begin{equation}
    \hat{J}^j_{i,k}(c) := \sum_{\tau=k}^{k+N-1} \ell_i\big(x^{j,c}_{\tau|k}, u^{j,c}_{\tau|k}\big) + Q^{j_r}_{i,k_r}.
    \label{eq:candidate_cost}
\end{equation}
Let $J^{B,j}_{i,k}$ denote player $i$'s remaining cost along the retained backup. The rules below compare each candidate with the backup through these quantities.

\textit{Admissibility rule (individual non-degradation):} A candidate is admissible if it does not increase any player's remaining cost relative to the backup:
\begin{equation}
    \mathcal{C}^j_{k,\mathrm{adm}} := \big\{ c \in \mathcal{C}^j_k \;:\; \hat{J}^j_{i,k}(c) \le J^{B,j}_{i,k}, \ \forall i \in \mathcal{M} \big\}.
    \label{eq:admissibility}
\end{equation}
This rule underlies the non-degradation guarantee of Theorem~2.

\textit{Selection rule:} Among admissible candidates, we consider two alternatives.

\emph{(a) Weighted cost.} The players minimize a weighted sum of their costs,
\begin{equation}
    c^\star \in \arg\min_{c \in \mathcal{C}^j_{k,\mathrm{adm}}} \sum_{i=1}^{M} w_i \hat{J}^j_{i,k}(c), \qquad w_i \ge 0, \ \sum_{i=1}^M w_i = 1.
    \label{eq:weighted_selection}
\end{equation}

\emph{(b) Nash bargaining.} Following~[14], each candidate is evaluated by the improvements it offers the players relative to a disagreement point. Here, disagreement means continuing along the current backup, so
\begin{equation}
    d_i := J^{B,j}_{i,k}, \qquad \Delta_i(c) := d_i - \hat{J}^j_{i,k}(c).
\end{equation}
Admissibility guarantees $\Delta_i(c) \ge 0$ for all $i$. The standard Nash product $\prod_i \Delta_i(c)$ rewards improvements for all players without requiring them to be equal. However, if some player cannot improve under any admissible candidate, the product is identically zero and cannot distinguish candidates that benefit the other players. We therefore use the regularized criterion
\begin{equation}
    c^\star \in \arg\max_{c \in \mathcal{C}^j_{k,\mathrm{adm}}} \prod_{i=1}^{M} \big(\Delta_i(c) + \epsilon\big), \qquad \epsilon > 0.
    \label{eq:nash_selection}
\end{equation}
The regularization affects only the ranking of admissible candidates. It does not relax admissibility, so candidates that increase any player's cost remain excluded. Because the disagreement point is the current backup, it is updated whenever a candidate is accepted or the backup is advanced.

\begin{definition}[Adopted Candidate]
\label{def:adopted_candidate}
Let $c \in \mathcal{C}^j_k$ be a candidate with terminal occurrence $q_r = (j_r, k_r)$ and predicted trajectory $\{x^{j,c}_{\tau|k}, u^{j,c}_{\tau|k}\}_{\tau=k}^{k+N-1}$, $x^{j,c}_{k+N|k}$. The candidate $c$ is \emph{adopted} at time $k$ of iteration $j$ if:
\begin{enumerate}
    \item it is \emph{verified primal feasible}, i.e., its predicted trajectory satisfies the dynamics~\eqref{eq:mig_dynamics}, the initial condition~\eqref{eq:mig_initial}, the private constraints~\eqref{eq:mig_private}, the shared constraints~\eqref{eq:mig_shared}, and the terminal condition~\eqref{eq:mig_terminal};
    \item it is admissible, i.e., $c \in \mathcal{C}^j_{k,\mathrm{adm}}$; and
    \item it is chosen by the agreed selection rule.
\end{enumerate}
Upon adoption, the retained backup is replaced by the \emph{complete trajectory} of $c$, obtained by following its predicted trajectory from $x^j_k$ to $x^{j_r}_{k_r}$ and then the stored trajectory $\{x^{j_r}_t, u^{j_r}_t\}_{t=k_r}^{T^{j_r}}$ from $x^{j_r}_{k_r}$ to $x_F$. If no candidate satisfies 1)--3), no candidate is adopted and the current backup is retained.
\end{definition}

\begin{proposition}[Feasibility of Adopted Candidates]
\label{prop:adopted_feasible}
If a candidate $c$ is adopted at time $k$ of iteration $j$, then its complete trajectory is a feasible joint trajectory from $x^j_k$ to $x_F$, i.e., it satisfies the dynamics, private constraints, and shared constraints at every time step. Moreover, its cost for player $i$ equals $\hat{J}^j_{i,k}(c)$ in~\eqref{eq:candidate_cost}.
\end{proposition}

\begin{proof}
By condition 1) of Definition~\ref{def:adopted_candidate}, the predicted trajectory of $c$ starts at $x^j_k$, satisfies the dynamics, private constraints, and shared constraints over the horizon, and ends at $x^{j,c}_{k+N|k} = x^{j_r}_{k_r}$. Since $q_r \in \mathcal{D}^{j-1}$ belongs to a successful iteration, the stored trajectory from $x^{j_r}_{k_r}$ satisfies the same constraints and reaches $x_F$. The two segments share the state $x^{j_r}_{k_r}$, so their concatenation is a feasible joint trajectory from $x^j_k$ to $x_F$. Its cost for player $i$ is the predicted stage cost over the horizon plus the stored cost-to-go $Q^{j_r}_{i,k_r}$ in~\eqref{eq:sampled_cost_to_go}, which is $\hat{J}^j_{i,k}(c)$.
\end{proof}

\begin{assumption}[Synchronous Decisions] \label{ass2}
All players share the trajectory database $\mathcal{D}^{j-1}$ and the current backup, apply the admissibility and selection rules to the same candidate set with a common deterministic tie-breaking rule, and adopt a new candidate at the same time step. Communication is reliable and synchronous.
\end{assumption}

\subsection{Practical Solution Considerations}

The learning problem combines discrete terminal selection with continuous allocation and equilibrium selection. Let $\Phi$ denote the cooperative criterion in minimization form, such as the negative regularized Nash product. Writing $c=(\boldsymbol\delta,\mathbf z_k^{j,\star})$, the problem can be written as
\begin{equation}
\begin{aligned}
    \min_{\boldsymbol\delta,\mathbf a,\mathbf z_k^{j,\star}}
    \quad &
    \Phi\!\left(
        \hat J_{1,k}^j(c),\ldots,\hat J_{M,k}^j(c)
    \right)
    \\
    \text{s.t.}\quad &
    \sum_{r=1}^{n_s}\delta_r=1,
    \qquad \delta_r\in\{0,1\},
    \\
    &
    \sum_{i=1}^{M}a_i=1,
    \qquad a_i>0,
    \\
    &
    \mathbf z_k^{j,\star}
    \in\operatorname{GNE}(\boldsymbol\delta,\mathbf a),
    \\
    &
    \hat J_{i,k}^j(c)\le J_{i,k}^{B,j},
    \qquad \forall i\in\mathcal M.
\end{aligned}
\label{eq:bilevel_learning_problem}
\end{equation}
Here, $\operatorname{GNE}(\boldsymbol\delta,\mathbf a)$ denotes equilibria of the selected terminal-constrained game admitting shared-constraint multipliers $\sigma_i=a_i\sigma$. In practice, we sample the dual variables scaling factors using a finite grid $\mathcal A_{\mathrm{grid}} =\{\mathbf a^{(1)},\ldots,\mathbf a^{(n_a)}\}$ satisfying \eqref{eq:scalar_normalization}. For each pair $(q_r,\mathbf a^{(p)})\in \mathcal D^{j-1}\times\mathcal A_{\mathrm{grid}}$, we formulate \eqref{eq:scaled_kkt} as a mixed complementarity problem (MCP) and solve it using PATH \cite{Ferris1999}. Candidates are checked for primal feasibility and individual admissibility, and the accepted candidate minimizing $\Phi$ is selected. Its prefix and stored suffix become the new backup, whose first input is applied. If no candidate is accepted, the existing backup is retained. With one MCP attempt per pair, each update requires at most $n_s n_a$ solves. These problems can be solved in parallel. A finer grid samples the allocation space more densely but increases computation. The present implementation is a proof of concept and is not intended for real-time execution. Nonnegative stage costs also permit pruning before candidate generation. If
\begin{equation}
    Q_{i,k_r}^{j_r}>J_{i,k}^{B,j}
\end{equation}
for any player $i$, every candidate using occurrence $q_r$ violates individual non-degradation. That occurrence can therefore be skipped at the current update.
 
\section{Theoretical Guarantees} \label{sec:theoretical_guarantees}

We establish recursive feasibility and individual non-degradation by extending the LMPC continuation argument \cite{RosoliaBorrelli2018} to joint trajectories. The key object is an explicitly retained feasible backup: an adopted candidate replaces it with a predicted prefix concatenated with a stored successful suffix. Otherwise,
the current backup is followed. This construction preserves a feasible continuation independently of whether a new equilibrium candidate can be computed. Let $J_{i,k}^{\mathrm B,j}$ denote player $i$'s remaining cost under the retained backup at time $k$ of iteration $j$. At the beginning of iteration $j$, the previous successful trajectory initializes the backup, so that
\begin{equation}
    J_{i,0}^{\mathrm B,j}=\mathcal J_i^{j-1}.
    \label{eq:initial_backup_cost}
\end{equation}
For a candidate connecting to occurrence $q_r=(j_r,k_r)$, the predicted prefix and stored suffix define a complete continuation whose cost is $\hat J_{i,k}^j(c)$ in \eqref{eq:candidate_cost}.

\begin{theorem}[Recursive Feasibility]
\label{thm:recursive_feasibility}
Suppose that an initial feasible joint trajectory from $x_S$ to $x_F$ is available. Assume that the system dynamics are deterministic and exactly known, and that the dynamics and constraints are unchanged across iterations.  Then the proposed learning scheme is recursively feasible, i.e., at every time step of every iteration a feasible continuation from the current joint state to $x_F$ is available. \end{theorem} 

\begin{proof}
The proof follows the standard shift-and-concatenation argument underlying LMPC recursive feasibility \cite{RosoliaBorrelli2018}, extended here to joint trajectories of the dynamic game. At time $k=0$ of the first learning iteration, the initially available successful trajectory provides a feasible backup from $x_S$ to $x_F$. Suppose a feasible backup from $x_k^j$ to $x_F$ is available at time $k$ of iteration $j$. If no new candidate is adopted, the  backup is retained. If a new candidate $c$ is adopted, the backup is replaced by the complete trajectory of $c$, which is a feasible trajectory from $x^j_k$ to $x_F$ by Proposition~\ref{prop:adopted_feasible}. In both cases, the first input of the resulting backup is applied. Since the dynamics are deterministic and exactly known, the realized state $x^j_{k+1}$ coincides with the next state of the backup, so its remaining portion is a feasible continuation from $x^j_{k+1}$ to $x_F$. The claim follows by induction over time steps and iterations.
\end{proof}

\begin{theorem}[Individual Non-Degradation]
\label{thm:nondegradation}
Suppose the conditions of Theorem~\ref{thm:recursive_feasibility} hold. Furthermore, suppose that a new candidate $c$ is adopted only if
\begin{equation}
    \hat J_{i,k}^j(c)
    \leq
    J_{i,k}^{\mathrm B,j},
    \qquad
    \forall i\in\mathcal M,
    \label{eq:theorem_acceptance}
\end{equation}
where $J_{i,k}^{\mathrm B,j}$ is the remaining cost of the currently retained feasible backup. Then, for every iteration $j$ completed in finite time, the realized cost of every player satisfies,
\begin{equation}
    \mathcal J_i^j
    \leq
    \mathcal J_i^{j-1},
    \qquad
    \forall i\in\mathcal M.
\end{equation}
\end{theorem}

\begin{proof}
The argument extends the LMPC iteration-cost monotonicity construction of \cite{RosoliaBorrelli2018} to player-specific costs in the cooperative dynamic-game setting. For player $i$, define the cost accumulated during iteration $j$ up to, but not including, time $k$ as
\begin{equation}
    S_{i,k}^j
    :=
    \sum_{\tau=0}^{k-1}
    \ell_i(x_\tau^j,u_\tau^j).
\end{equation}
Consider the sum of the cost already incurred and the remaining cost of the currently retained backup,
\begin{equation}
    S_{i,k}^j + J_{i,k}^{\mathrm B,j}.
    \label{eq:accumulated_plus_backup}
\end{equation}

Suppose first that no new candidate is adopted at time $k$. The first input of the current backup is applied and the backup is shifted by one time step. Therefore,
\begin{align}
    J_{i,k+1}^{\mathrm B,j}
    &=J_{i,k}^{\mathrm B,j}-\ell_i(x_k^j,u_k^j),\\
    S_{i,k+1}^j
    &=S_{i,k}^j+\ell_i(x_k^j,u_k^j).
\end{align}
Consequently,
\begin{equation}
    S_{i,k+1}^j + J_{i,k+1}^{\mathrm B,j}
    =
    S_{i,k}^j + J_{i,k}^{\mathrm B,j}.
\end{equation}

Suppose instead that a candidate $c$ is adopted.
Its predicted prefix and stored suffix become the new
backup. By \eqref{eq:theorem_acceptance}, applying its
first input gives
\begin{equation}
    J_{i,k+1}^{\mathrm B,j}
    =\hat J_{i,k}^j(c)-\ell_i(x_k^j,u_k^j)
    \le J_{i,k}^{\mathrm B,j}-\ell_i(x_k^j,u_k^j).
\end{equation}
Hence,
\begin{equation}
    S_{i,k+1}^j + J_{i,k+1}^{\mathrm B,j}
    \leq
    S_{i,k}^j + J_{i,k}^{\mathrm B,j}.
\end{equation}

Thus, in either case, $S_{i,k}^j+J_{i,k}^{\mathrm B,j}$ is non-increasing along the iteration. At the beginning of iteration $j$, using
\eqref{eq:initial_backup_cost},
\begin{equation}
    S_{i,0}^j+J_{i,0}^{\mathrm B,j}
    =
    \mathcal J_i^{j-1}.
\end{equation}
If iteration $j$ completes at time $T^j$, any remaining
backup suffix ends at $x_F$ and has nonnegative stage
costs by Assumption~1. Its remaining cost therefore
satisfies $J_{i,T^j}^{\mathrm B,j}\ge V_i(x_F)$.
Combining this with the preceding inequalities gives
\begin{equation}
    \mathcal J_i^j
    =S_{i,T^j}^j+V_i(x_F)
    \le S_{i,T^j}^j+J_{i,T^j}^{\mathrm B,j}
    \le \mathcal J_i^{j-1},
\end{equation}
for every player $i\in\mathcal M$.
\end{proof}

\begin{remark}
The guarantees of Theorems~\ref{thm:recursive_feasibility} and \ref{thm:nondegradation} do not require exhaustive exploration of the GNE set or a globally optimal solution of the outer selection problem. The discretization of the scaling parameters affects which improving candidates can be discovered, but not recursive feasibility or individual non-degradation. The individual non-degradation argument allows non-separable objectives and follows from additive costs and player-wise acceptance \eqref{eq:theorem_acceptance}, with Assumption \ref{ass1} ensuring that discarding any remaining suffix at completion cannot increase the cost. Under Assumption \ref{ass1}, finite-time completion additionally follows if $\sum_i\ell_i(x,u)\ge\rho(x)$ for every feasible $(x,u)$, where $\rho(x_F)=0$ and $\rho(x)\ge c>0$ whenever $x\neq x_F$. Indeed, every backup has remaining cost at least $V_i(x_F)$, so the proof of Theorem~\ref{thm:nondegradation} gives
\[
    kc \le \sum_i S_{i,k}^j
    \le \sum_i\bigl(\mathcal J_i^{j-1}-V_i(x_F)\bigr)
\]
for every $k$ before completion. Since the right-hand side is finite, infinitely many non-terminal steps are impossible. Nonnegative stage costs also justify the terminal-state pruning rule. For nonconvex games, the recursive-feasibility and non-degradation guarantees apply to any computed KKT/MCP candidate that is verified to satisfy the primal constraints and the acceptance condition; certification of global best-response optimality is not required for these two guarantees. If no acceptable candidate is found, the existing feasible backup is retained. 
\end{remark}

%===============================================================================
\section{Numerical Example and Results}

\subsection{Three-Player Cooperative Motion Planning}
This example is based on the example from \cite{RosoliaBorrelli2018}. We consider three agents repeatedly moving between initial and final states in a shared planar environment containing an elliptical obstacle. Each agent has its own target and cost function, while collision avoidance and shared bounds on speed and longitudinal actuation couple the players' decisions. The example uses a nonlinear unicycle model with heading as a state and heading rate as a control input. For player $i\in\{1,2,3\}$, the state and input are
\begin{equation}
    x_i = [p_{x,i},\;p_{y,i},\;v_i,\;\psi_i]^\top,
    \qquad
    u_i = [a_i,\;\omega_i]^\top,
\end{equation}
where $v_i$ is the forward speed, $\psi_i$ is the heading, $a_i$ is the commanded longitudinal acceleration, and $\omega_i$ is the heading rate. The continuous-time dynamics include quadratic drag:
\begin{equation}
\begin{aligned}
    \dot p_{x,i} &= v_i\cos\psi_i, &
    \dot p_{y,i} &= v_i\sin\psi_i,\\
    \dot v_i &= a_i-c_Dv_i^2, &
    \dot\psi_i &= \omega_i,
\end{aligned}
\label{eq:numerical_unicycle_dynamics}
\end{equation}
with $c_D=0.25\,\mathrm{m}^{-1}$. The dynamics are discretized by fourth-order Runge-Kutta integration with sampling time $\Delta t=0.25\,\mathrm{s}$ and piecewise-constant inputs.

The private constraints are
\begin{equation}
\begin{aligned}
    v_i\geq0,\qquad
    |\omega_i|\leq\frac{\pi}{2}\,\mathrm{rad/s},\\
    \frac{(p_{x,i}-1)^2}{6^2}
    +\frac{p_{y,i}^2}{2^2}\geq1,
\end{aligned}
\label{eq:numerical_private_constraints}
\end{equation}
where the position coordinates in the obstacle constraint are expressed in meters. Thus, the obstacle is centered at $(1,0)\,\mathrm{m}$ and has semi-axes $6\,\mathrm{m}$ and $2\,\mathrm{m}$. The players are coupled through the shared constraints
\begin{equation}
\begin{aligned}
    \sum_{i=1}^{3}v_i^2 \leq30[m/s]^2, \quad   \sum_{i=1}^{3}a_i^2 \leq 5[m/s^2]^2,\\
    \|p_i-p_j\|_2^2 \geq d_{\mathrm{sep}}^2,\quad i\neq j, \quad d_{\mathrm{sep}}=2[m]
\end{aligned}
\label{eq:numerical_collision_constraint}
\end{equation}
where $p_i=[p_{x,i},p_{y,i}]^\top$. The speed bound is a joint quadratic budget. These budgets create interactions even when collision avoidance is inactive. The obstacle and pairwise separation constraints make the feasible set nonconvex, in addition to the nonlinear dynamics. Constraints are imposed at the discrete prediction nodes. The initial and target positions and headings are listed in Table~\ref{tab:numerical_configurations}. Headings are shown in degrees for readability and expressed in radians in the dynamics.
\begin{table}[t]
    \centering
    \caption{Initial and target positions and headings.}
    \label{tab:numerical_configurations}
    \small
    \setlength{\tabcolsep}{5pt}
    \begin{tabular}{ccccc}
        \toprule
        Player & State & Position [m] & Velocity[m/s] & Heading \\
        \midrule
        1 & Initial & $(0,-9)$   & 0.25   & $135^\circ$ \\
          & Target  & $(0,9)$    & 0  & $90^\circ$  \\
        \midrule
        2 & Initial & $(4.5,4.5)$   & 0.25 & $180^\circ$ \\
          & Target  & $(-9,-9)$    & 0  & $-135^\circ$ \\
        \midrule
        3 & Initial & $(-4.5,4.5)$   & 0.25 & $45^\circ$  \\
          & Target  & $(9,-9)$     & 0  & $-90^\circ$ \\
        \bottomrule
    \end{tabular}
\end{table}

Each player uses the stage cost
\begin{equation}
    \ell_i(x_i,u_i)
    =(x_i-x_{i,F})^\top Q(x_i-x_{i,F})
    +u_i^\top R u_i+\rho_i(x_i),
    \label{eq:numerical_stage_cost}
\end{equation}
where
\[
    Q=\operatorname{diag}(0.1,0.1,0,0),\qquad
    R=\operatorname{diag}(0.1,0.1),
\]
and
\begin{equation}
    \rho_i(x_i)=
    \begin{cases}
        \Delta t, & x_i\neq x_{i,F},\\
        0, & \text{otherwise},
    \end{cases}
    \label{eq:numerical_time_cost}
\end{equation}
The cost penalizes position error, control effort and elapsed time outside the target joint state. Speed and heading are not directly penalized. $\rho_i(x_i)$ is part of the cost function which is solved by the dynamic game numerical solver even though it is discontinuous.

\subsection{Numerical Solution}

The prediction horizon is $N=5$, corresponding to $1.25\,\mathrm{s}$.
Learning is initialized with a feasible trajectory that completes the task. Each sampled terminal occurrence is stored together with its joint continuation, which provides the backup plan. Let $\boldsymbol\alpha=(\alpha_1,\alpha_2,\alpha_3)$ denote the multiplier allocation, with $A_i=\alpha_i I$ and $\alpha_1+\alpha_2+\alpha_3=1$. The multiple-equilibrium configurations use $16$ allocations:
\begin{equation}
\begin{aligned}
\mathcal A_{\mathrm{grid}}
={}\operatorname{Perm}\big\{&(\tfrac13,\tfrac13,\tfrac13),(0.50,0.25,0.25),\\&(0.80,0.10,0.10), 
(0.35,0.35,0.30),\\&(0.40,0.40,0.20), (0.40,0.30,0.30)
\big\}
\end{aligned}
\label{eq:numerical_alpha_grid}
\end{equation}
where $\operatorname{Perm}$ includes all distinct permutations of each
listed triple. Allocations remain fixed over the prediction horizon
and need not yield distinct equilibria. The players' KKT conditions are constructed using CasADi and assembled into an MCP, which is solved with PATH \cite{Ferris1999} with convergence tolerance $10^{-4}$. Because the game is nonconvex, MCP convergence identifies a first-order equilibrium candidate and does not certify global best-response optimality. The implementation is available in \cite{PustilnikLDGCode}. We compare three configurations:
\begin{enumerate}
    \item \textbf{Multiple GNE candidates with Nash bargaining (M-NB):}
    explore $\mathcal A_{\mathrm{grid}}$ and select an admissible candidate
    by maximizing the regularized Nash bargaining product.
    \item \textbf{Normalized GNE candidates with Nash bargaining (N-NB):}
    use only $\boldsymbol\alpha=(1/3,1/3,1/3)$, corresponding to equal
    shared-constraint multipliers, with the same bargaining rule.
    \item \textbf{Multiple GNE candidates with equally weighted cost
    (M-WS):} explore $\mathcal A_{\mathrm{grid}}$ and minimize
    $\frac13\sum_{i=1}^{3}\hat J_{i,k}^j(c)$ over admissible candidates $c$.
\end{enumerate}

The disagreement costs are initialized from the preceding iteration and updated to reflect each newly accepted plan and its remaining continuation. Among admissible
candidates, the selected plan maximizes the Nash bargaining product
\begin{equation}
    \prod_{i=1}^{3}
    \left(\max\{d_{i,k}^j-\hat J_{i,k}^j(c),0\}+\epsilon\right),
    \qquad \epsilon=10^{-6}.
    \label{eq:numerical_nash_product}
\end{equation}
The individual acceptance condition is checked before evaluating the product: clipping negative improvements does not make a degrading candidate admissible.

\subsection{Numerical Results}
\label{sec:numerical_results}

Figure~\ref{fig:learning_trajectories} illustrates the evolution of the
joint trajectory under M-NB. Successive iterations modify the players'
routes around the elliptical obstacle, with the final displayed iteration yielding costs of $450.78$, $449.78$, and $648.36$ for Players~1-3, respectively. Figure~\ref{fig:selection_costs} compares each player's iteration cost across the three configurations. All exhibit a pronounced initial decrease, followed by smaller improvements. Their relative performance changes over the learning iterations, as quantified in Table~\ref{tab:selection_costs}.

\begin{figure*}[t]
    \centering
    \includegraphics[width=\textwidth]{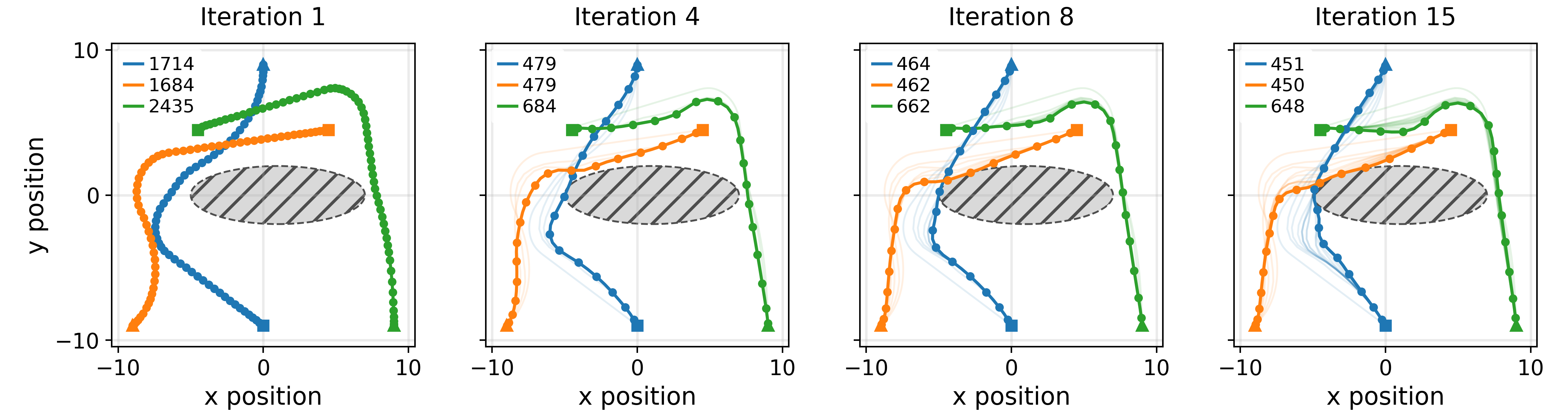}
    \caption{Trajectory evolution under M-NB at iterations $1$, $4$, $8$, and $15$, with earlier trajectories shown faintly. Blue, orange, and green identify Players~1-3. Legend values give their iteration costs.}
    \label{fig:learning_trajectories}
\end{figure*}

\begin{figure}[t]
    \centering
    \includegraphics[width=\columnwidth]{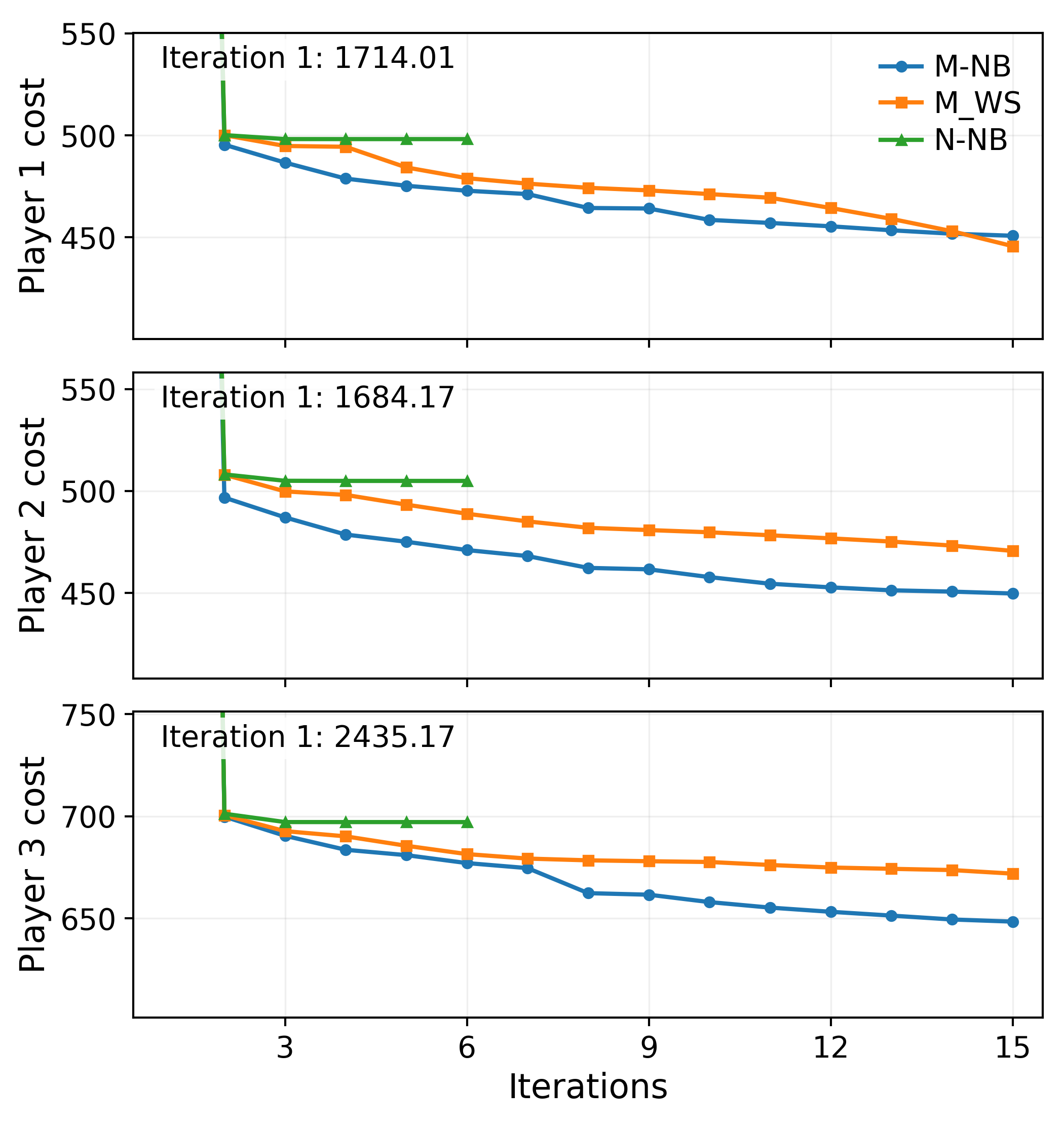}
    \caption{Player costs across iterations for M-NB, M-WS, and N-NB. Initial costs are annotated because they lie outside the displayed ranges. N-NB terminates at iteration $6$, with costs unchanged from iteration $5$ to numerical tolerance.}
    \label{fig:selection_costs}
\end{figure}

\begin{table}[t]
    \centering
    \caption{Player and aggregate costs at initialization, iteration $5$, and the final iteration of each method. Bold entries indicate the lowest cost within each comparison group.}
    \label{tab:selection_costs}
    \setlength{\tabcolsep}{3pt}
    \footnotesize
    \begin{tabular}{clrrrr}
        \hline
        $j$ & Method & $\mathcal J_1^j$ & $\mathcal J_2^j$
            & $\mathcal J_3^j$ & $\sum_i\mathcal J_i^j$ \\
        \hline
        1 & All & 1714.01 & 1684.17 & 2435.17 & 5833.35 \\
        \hline
        5 & M-NB & \textbf{475.29} & \textbf{475.15}
            & \textbf{680.94} & \textbf{1631.37} \\
        5 & N-NB & 498.16 & 504.96 & 697.14 & 1700.26 \\
        5 & M-WS & 484.27 & 493.34 & 685.54 & 1663.14 \\
        \hline
        15 & M-NB & 450.78 & \textbf{449.78}
            & \textbf{648.36} & \textbf{1548.92} \\
        15 & N-NB & 498.16 & 504.96 & 697.14 & 1700.26 \\
        15 & M-WS & \textbf{445.69} & 470.66
            & 671.87 & 1588.22 \\
        \hline
    \end{tabular}
\end{table}

All methods start from the same feasible trajectory, with aggregate cost 5833.35. At iteration 5, M-NB achieves the lowest aggregate cost, 1631.37, improving on M-WS and N-NB by 1.91\% and 4.05\%, respectively. Its costs are also lower for every individual player at this checkpoint. At iteration 15, M-NB retains the lowest aggregate cost, 1548.92, which is 2.47\% below M-WS. Neither method dominates the other across all players: M-NB incurs a 1.14\% higher cost for Player~1 relative to M-WS, while reducing the costs of Players~2 and~3 by 4.44\% and 3.50\%, respectively. Whereas M-NB reduces the costs of Players~2 and~3 by 4.44\% and 3.50\%, respectively, relative to M-WS. The final comparison therefore exhibits a trade-off between individual outcomes. N-NB terminates at iteration 6 because its costs remain unchanged from iteration 5 to numerical tolerance, with a final aggregate cost of 1700.26. M-NB continues to improve through iteration 15, reaching an aggregate cost 8.90\% below N-NB's final value. These results illustrate how allocation exploration and cooperative selection affect performance in this example, without establishing general superiority of either selection rule. Because accepted trajectories populate the terminal memory, selection decisions also affect the candidates available in subsequent iterations. The reported results do not establish asymptotic convergence or global optimality.
%===============================================================================
\section{Conclusions}
This paper presented a learning framework for cooperative iterative dynamic games with shared constraints. Stored joint trajectories provide a sampled safe set and
player-specific costs-to-go for cooperative selection among terminal joint states and equilibrium candidates. Retaining a feasible continuation ensures recursive feasibility, while player-wise acceptance guarantees non-degradation across completed iterations under the stated assumptions. The framework accommodates nonlinear and nonconvex games through verified feasible candidates. A three-player motion-planning example illustrates the effects of multiplier exploration and cooperative selection. Future work will address computational efficiency and robustness to model uncertainty.

%===============================================================================
\bibliographystyle{IEEEtran}
\bibliography{main}

\end{document}